\documentclass{article}

\PassOptionsToPackage{numbers, compress}{natbib}

\usepackage[final]{neurips_2018}

\usepackage[utf8]{inputenc} 
\usepackage[T1]{fontenc}    
\usepackage{hyperref}       
\usepackage{url}            
\usepackage{booktabs}       
\usepackage{amsfonts}       
\usepackage{nicefrac}       
\usepackage{microtype}      
\usepackage{amsmath}

\usepackage{nameref,zref-xr}
\usepackage{hyperref}

\zxrsetup{toltxlabel}
\zexternaldocument*{appendix}
\numberwithin{equation}{section}

\usepackage{geometry}
\usepackage{algorithm}
\usepackage{algpseudocode}
\usepackage{algorithmicx}

\usepackage[utf8]{inputenc}
\usepackage{amsfonts}
\usepackage{enumerate}
\usepackage{graphicx}
\usepackage{xspace}
\usepackage{stmaryrd}
\usepackage{ae}
\usepackage{enumerate}
\usepackage{bbm}
\usepackage{bm}
\usepackage{mathtools}
\usepackage{amsthm}

\usepackage{array}
\usepackage{makecell}
\usepackage{siunitx}
\usepackage{subfigure}
\usepackage{tcolorbox}

\newcommand{\RR}[0]{\mathbb{R}}
\newcommand{\ip}[2]{\langle #1, #2 \rangle}

\newcommand{\Wcal}{\mathcal{W}}

\newcommand{\Xcal}{\mathcal{X}}

\def\l({\left(}
\def\r){\right)}
\def\bl({\Big(}
\def\br){\Big)}
\def\beq{\begin{equation}}
\def\eeq{\end{equation}}

\newtheorem{assumption}{Assumption}

\usepackage{tikz,pgfplots}
\def\x{{\mathbf x}}
\def\y{{\mathbf y}}

\def\X{{\mathbf X}}
\def\Y{{\mathbf Y}}
\def\n{{\mathbf n}}

\def\L{{\cal L}}
\def\RR{{\mathbb R}}

\def\PP{{ \mathbb P }}
\def\EE{{ \mathbb E }}

\def\wt{{w_{t}}}
\def\wt1{{w_{t+1}}}
\def\at1{{a_{t+1}}}

\def\Bbf{{ \mathbf B }}

\def\hstar{{  h^\star  }}

\def\L_h{{ \mathcal{L}_h }}

\newcommand{\drm}{\mathrm{d}}

\def\nh{{  \nabla h }}
\def\n2h{{  \nabla^2 h }}
\def\nhstar{{  \nabla h^\star  }}
\def\n2hstar{{  \nabla^2 h^\star  }}

\def\dt{{\drm t}}
\def\dBt{{ \drm  \mathbf B_t }}
\def\dmu{{ \drm \mu }}
\def\dnu{{ \drm \nu }}
\def\dV{{ e^{-V(\x)}\drm \x }}
\def\dW{{ e^{-W(\y)}\drm \y }}
\def\dTV{{ d_{\textup{TV}} }}

\def\nn{{ \nonumber }}

\newcommand{\ts}{\textsuperscript}

\newtheorem{theorem}{Theorem}
\newtheorem{lemma}{Lemma}
\newtheorem{remark}{Remark}
\newtheorem{example}{Example}

\def\hstar{{ h^\star }}

\def\n{{ \nabla }}

\def\drm{{ \mathrm{d} }}

\title{Mirrored Langevin Dynamics}

\author{
  Ya-Ping Hsieh  
  \And
  Ali Kavis \\
  \And
  Paul Rolland \\
  \And
  Volkan Cevher 
  \AND 
  \normalfont{Laboratory for Information and Inference Systems (LIONS),}\\
  \normalfont{EPFL, Lausanne, Switzerland}\\
  \texttt{\{ya-ping.hsieh, ali.kavis, paul.rolland, volkan.cevher\}@epfl.ch}
}

\begin{document}

\maketitle

\begin{abstract}
We consider the problem of sampling from \emph{constrained} distributions, which has posed significant challenges to both non-asymptotic analysis and algorithmic design. We propose a unified framework, which is inspired by the classical mirror descent, to derive novel first-order sampling schemes. We prove that, for a general target distribution with strongly convex potential, our framework implies the existence of a first-order algorithm achieving $\tilde{O}(\epsilon^{-2}d)$ convergence, suggesting that the state-of-the-art $\tilde{O}(\epsilon^{-6}d^5)$ can be vastly improved. With the important Latent Dirichlet Allocation (LDA) application in mind, we specialize our algorithm to sample from Dirichlet posteriors, and derive the first non-asymptotic $\tilde{O}(\epsilon^{-2}d^2)$ rate for first-order sampling. We further extend our framework to the mini-batch setting and prove convergence rates when only stochastic gradients are available. Finally, we report promising experimental results for LDA on real datasets.

\end{abstract}

\section{Introduction}

%
%

Many modern learning tasks involve sampling from a high-dimensional and large-scale distribution, which calls for algorithms that are scalable with respect to both the dimension and the data size. One approach \cite{welling2011bayesian} that has found wide success is to discretize the \textbf{Langevin Dynamics}:
\beq \label{eq:langevin_dynamics}
\drm\X_t = -\nabla V(\X_t)\dt + \sqrt{2} \dBt,
\eeq 
where $\dV$ presents a target distribution and $\Bbf_t$ is a $d$-dimensional Brownian motion. Such a framework has inspired numerous first-order sampling algorithms \cite{ahn2012bayesian, chen2014stochastic, ding2014bayesian, durmus2016stochastic, lan2016sampling, liu2016stochastic, patterson2013stochastic, simsekli2016stochastic}, and the convergence rates are by now well-understood for unconstrained and log-concave distributions \citep{cheng2017convergence, dalalyan2017user, durmus2018analysis}.  

However, applying \eqref{eq:langevin_dynamics} to sampling from \emph{constrained} distributions (i.e., when $V$ has a bounded convex domain) remains a difficult challenge. From the theoretical perspective, there are only two existing algorithms \cite{brosse17sampling, bubeck2015sampling} that possess non-asymptotic guarantees, and their rates are significantly worse than the unconstrained scenario under the same assumtions; \textit{cf.}, Table~\ref{tab:constrained_rates}. Furthermore, many important constrained distributions are inherently non-log-concave. A prominent instance is the \textbf{Dirichlet posterior}, which, in spite of the presence of several tailor-made first-order algorithms  \cite{lan2016sampling, patterson2013stochastic}, is still lacking a non-asymptotic guarantee.

In this paper, we aim to bridge these two gaps at the same time. For general constrained distributions with a strongly convex potential $V$, we prove the existence of a first-order algorithm that achieves the same convergence rates as if there is no constraint at all, suggesting the state-of-the-art $\tilde{O}(\epsilon^{-6}d^5)$ can be brought down to $\tilde{O}(\epsilon^{-2}d)$.
~When specialized to the important case of simplex constraint, we provide the first non-asymptotic guarantee for Dirichlet posteriors, $\tilde{O}(\epsilon^{-2}d^2R_0)$ for deterministic and $\tilde{O}\l(\epsilon^{-2}(Nd + \sigma^2)R_0\r)$ for the stochastic version of our algorithms; \textit{cf.}, \textbf{Example \ref{exp:dirichlet}} and \textbf{\ref{exp:diri_smld}} for the involved parameters.

Our framework combines ideas from the \textbf{Mirror Descent} \cite{beck2003mirror, nemirovsky1983problem} algorithm for optimization and the theory of Optimal Transport \cite{villani2008optimal}. Concretely, for constrained sampling problems, we propose to use the \emph{mirror map} to transform the target into an unconstrained distribution, whereby many existing methods apply. Optimal Transport theory then comes in handy to relate the convergence rates between the original and transformed problems. For simplex constraints, we use the entropic mirror map to design practical first-order algorithms that possess rigorous guarantees, and are amenable to mini-batch extensions.

The rest of the paper is organized as follows. We briefly review the notion of push-forward measures in Section \ref{sec:prelim}. In Section \ref{sec:amld}, we propose the \textbf{Mirrored Langevin Dynamics} and prove its convergence rates for constrained sampling problems. Mini-batch extensions are derived in Section \ref{sec:stochastic_mld}. Finally, in Section \ref{sec:experiment}, we provide synthetic and real-world experiments to demonstrate the empirical efficiency of our algorithms.
\subsection{Related Work} 
\noindent\textbf{First-Order Sampling Schemes with Langevin Dynamics:} There exists a bulk of literature on (stochastic) first-order sampling schemes derived from Langevin Dynamics or its variants \citep{ahn2012bayesian, brosse17sampling, bubeck2015sampling, chen2015convergence, cheng2017convergence, cheng2017underdamped, dalalyan2017user, durmus2018analysis, dwivedi2018log, luu2017sampling, patterson2013stochastic, welling2011bayesian}. However, to our knowledge, this work is the first to consider mirror descent extensions of the Langevin Dynamics.

The authors in \cite{ma2015complete} proposed a formalism that can, in principle, incorporate any variant of Langevin Dynamics for a given distribution $\dV$. The Mirrored Langevin Dynamics, however, is targeting the push-forward measure $\dW$ (see Section \ref{sec:amld_motivation}), and hence our framework is not covered in \cite{ma2015complete}.

For Dirichlet posteriors, there is a similar variable transformation as our entropic mirror map in \cite{patterson2013stochastic} (see the ``reduced-natural parametrization'' therein). The dynamics in \cite{patterson2013stochastic} is nonetheless drastically different from ours, as there is a position-dependent matrix multiplying the Brownian motion, whereas our dynamics has no such feature; see \eqref{eq:asymmetric_mirror_descent_dyanmics}.


\noindent\textbf{Mirror Descent-Type Dynamics for Stochastic Optimization:} Although there are some existing work on mirror descent-type dynamics for \emph{stochastic optimization} \citep{krichene2017acceleration, mertikopoulos2018convergence, raginsky2012continuous, xu2018accelerated}, we are unaware of any prior result on sampling. 





\section{Preliminaries}\label{sec:prelim}
\subsection{Notation}
In this paper, all Lipschitzness and strong convexity are with respect to the Euclidean norm $\|\cdot\|$. We use $\mathcal{C}^k$ to denote $k$-times differentiable functions with continuous $k$\ts{th} derivative. The Fenchel dual \cite{rockafellar2015convex} of a function $h$ is denoted by $\hstar$. Given two mappings $T, F$ of proper dimensions, we denote their composite map by $T\circ F$. For a probability measure $\mu$, we write $\X\sim \mu$ to mean that ``$\X$ is a random variable whose probability law is $\mu$''. 



\subsection{Push-Forward and Optimal Transport}
Let $\dmu = \dV$ be a probability measure with support $\Xcal \coloneqq \textup{dom}(V) = \{\x \in \RR^d\ |\ V(\x) < +\infty \}$, and $h$ be a convex function on $\Xcal$. Throughout the paper we assume:
\begin{assumption}\label{ass:mirror_map}
$h$ is closed, proper, $h \in \mathcal{C}^2 \text{, and } \nabla^2 h \succ 0$ on $\Xcal \subset \RR^d$.
\end{assumption}
\begin{assumption}\label{ass:distributions_moment}
All measures have finite second moments.
\end{assumption}
\begin{assumption}\label{ass:distributions_hausdorff}
All measures vanish on sets with Hausdorff dimension \cite{mandelbrot1983fractal} at most $d-1$.
\end{assumption}\vspace{-2mm}
The gradient map $\nabla h$ induces a new probability measure $\dnu \coloneqq \dW$ through $\nu(E) = \mu \l(\nabla h^{-1}(E) \r)$ for every Borel set $E$ on $\RR^d$. We say that $\nu$ is the \textbf{push-forward measure} of $\mu$ under $\nabla h$, and we denote it by $\nabla h \# \mu = \nu$. If $\X \sim \mu$ and $\Y\sim \nu$, we will sometimes abuse the notation by writing $\nh \# \X = \Y$ to mean $\nh \# \mu = \nu.$ 

If $\nh \# \mu = \nu$, the triplet $(\mu, \nu, h)$ must satisfy the Monge-Amp\`ere equation:
\beq \label{eq:Monge-Ampere}
e^{-V} = e^{-W \circ \nabla h} \det \nabla^2 h.
\eeq
Using $(\nabla h)^{-1}  = \nabla \hstar $ and $\nabla^2 h \circ \nabla \hstar = \nabla^2 \hstar^{-1}$, we see that \eqref{eq:Monge-Ampere} is equivalent to
\beq \label{eq:Monge-Ampere-dual}
e^{-W} = e^{-V \circ \nabla \hstar} \det \nabla^2 \hstar
\eeq
which implies $\nabla \hstar \# \nu = \mu$.


The 2-Wasserstein distance between $\mu_1$ and $\mu_2$ is defined by\footnote{In general, \eqref{eq:optimal_transportation} is ill-defined; see \cite{villani2003topics}. The validity of \eqref{eq:optimal_transportation} is guaranteed by McCann's theorem \citep{mccann1995existence} under \textbf{Assumption \ref{ass:distributions_moment}} and \textbf{\ref{ass:distributions_hausdorff}}. }
\beq  \label{eq:optimal_transportation}
\Wcal^2_2(\mu_1, \mu_2) \coloneqq \inf_{T: T\# \mu_1 = \mu_2} \int  \|  \x - T(\x) \|^2  \dmu_1(\x).
\eeq

\section{Mirrored Langevin Dynamics}\label{sec:amld}
This section demonstrates a framework for transforming constrained sampling problems into unconstrained ones. We then focus on applications to sampling from strongly log-concave distributions and simplex-constrained distributions, even though the framework is more general and future-proof.

\subsection{Motivation and Algorithm}\label{sec:amld_motivation}
We begin by briefly recalling the mirror descent (MD) algorithm for optimization. In order to minimize a function over a bounded domain, say $\min_{\x \in \Xcal} f(\x)$, MD uses a mirror map $h$ to transform the primal variable $\x$ into the dual space $\y \coloneqq \nh (\x)$, and then performs gradient updates in the dual: $\y^+ = \y - \beta \nabla f(\x)$ for some step-size $\beta$. The mirror map $h$ is chosen to adapt to the geometry of the constraint $\Xcal$, which can often lead to faster convergence \cite{nemirovsky1983problem} or, more pivotal to this work, an \textbf{unconstrained} optimization problem \cite{beck2003mirror}.

Inspired by the MD framework, we would like to use the mirror map idea to remove the constraint for sampling problems.
Toward this end, we first establish a simple fact \citep{villani2003topics}:
\begin{theorem}\label{thm:coupling}
Let $h$ satisfy \textbf{Assumption \ref{ass:mirror_map}}. Suppose that $\X\sim \mu$ and $\Y = \nh (\X)$. Then $\Y  \sim \nu \coloneqq \nabla h \# \mu$ and $\nhstar (\Y) \sim \mu$.
\end{theorem}
\begin{proof}
For any Borel set $E$, we have $\nu(E) = \PP \l( \Y \in E \r)  = \PP \l( \X \in \nabla h^{-1}(E) \r) =  \mu\l(\nabla h^{-1} (E) \r)$. Since $\nh$ is one-to-one, $\Y = \nh (\X)$ if and only if $\X = \nh^{-1}(\Y)= \nhstar(\Y)$.
\end{proof}

In the context of sampling, \textbf{Theorem \ref{thm:coupling}} suggests the following simple procedure: For any target distribution $\dV$ with support $\Xcal$, we choose a mirror map $h$ on $\Xcal$ satisfying \textbf{Assumption \ref{ass:mirror_map}}, and we consider the \textbf{dual distribution} associated with $\dV$ and $h$:
\beq \label{eq:dual_distribution}
\dW \coloneqq \nh \# \dV.
\eeq 
\textbf{Theorem \ref{thm:coupling}} dictates that if we are able to draw a sample $\Y$ from $\dW$, then $\nhstar(\Y)$ immediately gives a sample for the desired distribution $\dV$. Furthermore, suppose for the moment that $\textup{dom}(\hstar) = \RR^d$, so that $\dW$ is unconstrained. Then we can simply exploit the classical Langevin Dynamics \eqref{eq:langevin_dynamics} to efficiently take samples from $\dW$. 

The above reasoning leads us to set up the \textbf{Mirrored Langevin Dynamics} (MLD):
\begin{tcolorbox}
\beq \label{eq:asymmetric_mirror_descent_dyanmics}
\textbf{MLD} \equiv   \left\{
                \begin{array}{ll}
                  \drm\Y_t = -(\nabla W\circ \nabla h)(\X_t) \dt + \sqrt{2} \dBt\\
                  \X_t = \nabla \hstar (\Y_t)
                \end{array}.
        \right.
\eeq
\end{tcolorbox}
Notice that the stationary distribution of $\Y_t$ in MLD is $\dW$, since $\drm \Y_t$ is nothing but the Langevin Dynamics \eqref{eq:langevin_dynamics} with $\nabla V \leftarrow \nabla W$. As a result, we have $\X_t \rightarrow \X_\infty \sim \dV$.

%
%
%
%
%

Using \eqref{eq:Monge-Ampere}, we can equivalently write the $\drm\Y_t$ term in \eqref{eq:asymmetric_mirror_descent_dyanmics} as \beq \nn
\drm\Y_t = -  \nabla^2 h(\X_t)^{-1}\Big(\nabla V(\X_t) + \nabla \log\det \nabla^2h(\X_t)  \Big)\dt + \sqrt{2} \dBt.\eeq
In order to arrive at a practical algorithm, we then discretize the MLD, giving rise to the following equivalent iterations:
\begin{tcolorbox}\vspace{-3mm}
\beq \label{eq:asymmetric_mirror_dynamics_discrete}
\y^{t+1} - \y^t = \left\{
                \begin{array}{ll}
                    - \beta^t \nabla W  (\y^t)+ \sqrt{2\beta^t}  \bm{\xi}^t \\
                  -  \beta^t\nabla^2 h(\x^t)^{-1}\Big(\nabla V(\x^t) + \nabla \log\det \nabla^2h(\x^t)  \Big) + \sqrt{2\beta^t}  \bm{\xi}^t 
                \end{array}
              \right. 
\eeq
\end{tcolorbox}
where in both cases $\x^{t+1} = \nhstar (\y^{t+1})$, $\bm{\xi}^t$'s are i.i.d. standard Gaussian, and $\beta^t$'s are step-sizes. 
The first formulation in \eqref{eq:asymmetric_mirror_dynamics_discrete} is useful when $\nabla W$ has a tractable form, while the second one can be computed using solely the information of $V$ and $h$.



Next, we turn to the convergence of discretized MLD. Since $\drm \Y_t$ in \eqref{eq:asymmetric_mirror_descent_dyanmics} is the classical Langevin Dynamics, and since we have assumed that $W$ is unconstrained, it is typically not difficult to prove the convergence of $\y^t$ to $\Y_\infty \sim \dW$. However, what we ultimately care about is the guarantee on the primal distribution $\dV$. The purpose of the next theorem is to fill the gap between primal and dual convergence.

We consider three most common metrics in evaluating approximate sampling schemes, namely the 2-Wasserstein distance $\Wcal_2$, the total variation $\dTV$, and the relative entropy $D(\cdot \| \cdot)$. 

\begin{theorem}[Convergence in $\y^t$ implies convergence in $\x^t$]\label{thm:convergence_TV}
For any $h$ satisfying \textbf{Assumption \ref{ass:mirror_map}}, we have $d_{\textup{TV}}(\nabla h\# \mu_1, \nabla h\# \mu_2) = d_{\textup{TV}}(\mu_1, \mu_2) $ and $D(\nabla h\# \mu_1 \|  \nabla h\# \mu_2) = D(\mu_1 \| \mu_2)$. In particular, we have $d_{\textup{TV}}(\y^t, \Y_\infty) =d_{\textup{TV}}(\x^t, \X_\infty)$ and $D(\y^t \|  \Y_\infty) = D(\x^t \| \X_\infty)$ in \eqref{eq:asymmetric_mirror_dynamics_discrete}.

If, furthermore, $h$ is $\rho$-strongly convex: $\nabla^2 h \succeq \rho I$. Then $\Wcal_2(\x^t, \X_\infty)\leq \frac{1}{\rho}\Wcal_2(\y^t, \Y_\infty)$.
\end{theorem}
\begin{proof}\vspace{-4mm}
See \textbf{Appendix \ref{app:proof_amld_TV}}.
\end{proof}

\subsection{Applications to Sampling from Constrained Distributions}\label{subsec:applications}
We now consider applications of MLD. For strongly log-concave distributions with general constraint, we prove matching rates to that of unconstrained ones; see Section \ref{sec:strongly_log-concave}. In Section \ref{sec:simplex}, we consider the important case where the constraint is a {probability simplex}\footnote{More examples of mirror map can be found in \textbf{Appendix \ref{app:mirror_maps}}.}.

\subsubsection{Sampling from a strongly log-concave distribution with constraint}\label{sec:strongly_log-concave}
As alluded to in the introduction, the existing convergence rates for constrained distributions are significantly worse than their unconstrained counterparts; see Table \ref{tab:constrained_rates} for a comparison. 
\begin{table}[t]
\centering
\begin{tabular}{|l||*{4}{c|}}\hline
\makebox[5em]{\small Assumption} & \makebox[5em]{{$D(\cdot \| \cdot )$}} &\makebox[5em]{{$\mathcal{W}_2$}}&\makebox[5em]{$d_{\text{TV}}$}&\makebox[5em]{\small Algorithm} \\\hline \hline
\small$L I \succeq \nabla^2 V \succeq m I$ & {\small unknown} & {\small unknown} &  {\small$\tilde{O}\l(\epsilon^{-6}d^5 \r)$ } &\small {MYULA  \cite{brosse17sampling}} \\\hline
\small $L I \succeq \nabla^2 V \succeq 0$ & {\small unknown} & {\small unknown} &{\small$\tilde{O}\l(\epsilon^{-12}d^{12} \r)$} &\small {PLMC \cite{bubeck2015sampling}}  \\\hline  
\small $\nabla^2 V \succeq m I$ & {\small \color{blue}$\tilde{O}\l(\epsilon^{-1}d \r)$} & {\small \color{blue}$\tilde{O}\l(\epsilon^{-2}d \r)$} &{\small \color{blue}$\tilde{O}\l(\epsilon^{-2}d \r)$} &\small {\color{blue} MLD; this work}  \\\hline \hline
\makecell{\small$ L I \succeq \nabla^2 V \succeq m I$, \\ $V$ unconstrained} & {\small $\tilde{O}\l(\epsilon^{-1}d \r)$} & {\small $\tilde{O}\l(\epsilon^{-2}d \r)$} &{\small $\tilde{O}\l(\epsilon^{-2}d \r)$} &Langevin Dynamics \cite{cheng2017convergence, dalalyan2017theoretical,durmus2018analysis} \\\hline
\end{tabular}\caption{Convergence rates for sampling from $\dV$ with $\textup{dom}(V)$ bounded}\label{tab:constrained_rates}\vspace{-8mm}
\end{table}
%
The main result of this subsection is the existence of a ``good'' mirror map for \emph{arbitrary} constraint, with which the dual distribution $\dW$ becomes unconstrained:
\begin{theorem}[Existence of a good mirror map for MLD]\label{thm:existence_good_mirror_map}
Let $\drm \mu(\x) = e^{-V(\x)} \drm \x$ be a probability measure with bounded convex support such that $V \in \mathcal{C}^2$, $\nabla^2 V \succeq m I \succ 0$, and $V$ is bounded away from $+\infty$ in the interior of the support. Then there exists a mirror map $h \in \mathcal{C}^2$ such that the discretized MLD \eqref{eq:asymmetric_mirror_dynamics_discrete} yields
\beq  \nn
D\l( \x^T \| \X_\infty \r)=\tilde{O}\l(   \frac{d}{T} \r),  \quad \Wcal_2 \l( \x^T, \X_\infty \r) = \tilde{O}\l(  \sqrt{ \frac{d}{{T}} } \r), \quad d_{\textup{TV}} \l(\x^T, \X_\infty \r) = \tilde{O}\l( \sqrt{ \frac{d}{{T}} } \r).
\eeq
%
%
\end{theorem}
\begin{proof}\vspace{-4mm}
See \textbf{Appendix \ref{app:proof_existence}}.
\end{proof}
\begin{remark}
We remark that \textbf{Theorem \ref{thm:existence_good_mirror_map}} is only an existential result, not an actual algorithm. Practical algorithms are considered in the next subsection.
\end{remark}

\subsubsection{Sampling Algorithms on Simplex}\label{sec:simplex}
We apply the discretized MLD \eqref{eq:asymmetric_mirror_dynamics_discrete} to the task of sampling from distributions on the probability simplex $\Delta_{d} \coloneqq \{ \x \in \RR^{d}\ | \ \sum_{i=1}^{d}{x_i} \leq 1, x_i \geq 0 \}$, which is instrumental in many fields of machine learning and statistics.

On a simplex, the most natural choice of $h$ is the entropic mirror map \cite{beck2003mirror}, which is well-known to be 1-strongly convex:
\beq \label{eq:entropic_mirror}
h(\x) = \sum_{\ell =1}^d x_i \log x_\ell + \l(1-\sum_{\ell =1}^d{x_\ell} \r)\log\l(1-\sum_{\ell =1}^d{x_\ell}\r), \text{ where }\  0 \log 0 \coloneqq 0.
\eeq
In this case, the associated dual distribution can be computed explicitly.
\begin{lemma}[Sampling on a simplex with entropic mirror map] \label{lem:W_simplex}
Let $\dV$ be the target distribution on $\Delta_d$, $h$ be the entropic mirror map \eqref{eq:entropic_mirror}, and $\dW \coloneqq \nh \# \dV$. Then the potential $W$ of the push-forward measure admits the expression
\begin{align}
W(\y) 
&= V\circ \nhstar (\y)- \sum_{\ell=1}^d{y_\ell} + (d+1) \hstar (\y)\label{eq:formula_W}
\end{align}where $\hstar (\y) = \log \l(1+ \sum_{\ell=1}^d e^{y_\ell} \r)$ is the Fenchel dual of $h$, which is strictly convex and 1-Lipschitz gradient.  
\end{lemma}
\begin{proof}\vspace{-3mm}
See \textbf{Appendix \ref{app:W_simplex}}.
\end{proof}
Crucially, we have $\textup{dom}(\hstar) = \RR^d$, so that the Langevin Dynamics for $\dW$ is \textbf{unconstrained}.


Based on \textbf{Lemma \ref{lem:W_simplex}}, we now present the surprising case of the \emph{non-log-concave} Dirichlet posteriors, a distribution of central importance in topic modeling  \citep{blei2003latent}, for which the dual distribution $\dW$ becomes strictly \emph{log-concave}. 
\begin{example}[Dirichlet Posteriors]\label{exp:dirichlet}
\emph{Given parameters $\alpha_1, \alpha_2, ..., \alpha_{d+1} >0$ and observations $n_1, n_2, ..., n_{d+1}$ where $n_{\ell}$ is the number of appearance of category $\ell$, the probability density function of the Dirichlet posterior is 
\beq  \label{eq:dirichlet_distribution}
p(\x) = \frac{1}{C} \prod_{\ell=1}^{d+1} x_\ell^{n_\ell+\alpha_\ell-1}, \quad \x \in \textup{int}\l(\Delta_{d}\r)
\eeq
where $C$ is a normalizing constant and $x_{d+1} \coloneqq 1- \sum_{\ell=1}^d x_\ell$. The corresponding $V$ is
\beq \nn
V(\x) = -\log p(\x) = \log C - \sum_{\ell=1}^{d+1} (n_\ell+\alpha_\ell-1) \log x_\ell, \quad \x \in \textup{int}\l(\Delta_{d}\r). \nn
\eeq
The interesting regime of the Dirichlet posterior is when it is \textbf{sparse}, meaning the majority of the $n_\ell$'s are zero and a few $n_k$'s are large, say of order $O(d)$. It is also common to set $\alpha_\ell <1$ for all $\ell$ in practice. Evidently, $V$ is neither convex nor concave in this case, and no existing non-asymptotic rate can be applied. However, plugging $V$ into \eqref{eq:formula_W} gives
\beq \label{eq:dirichlet_dual}
W(\y) = \log C - \sum_{\ell=1}^d (n_\ell + \alpha_\ell)y_\ell + \l(\sum_{\ell=1}^{d+1} (n_\ell+\alpha_\ell) \r) \hstar (\y)
\eeq
which, magically, becomes strictly convex and $O(d)$-Lipschitz gradient \textbf{no matter what the observations and parameters are!} In view of \textbf{Theorem \ref{thm:convergence_TV}} and \citep[][\textbf{Corollary 7}]{durmus2018analysis}, one can then apply \eqref{eq:asymmetric_mirror_dynamics_discrete} to obtain an $\tilde{O}\l(\epsilon^{-2} {d^2} R_0\r)$ convergence in relative entropy, where $R_0 \coloneqq \Wcal^2_2(\y^0,\dW)$ is the initial Wasserstein distance to the target. } \hfill \qed
\end{example}


\section{Stochastic Mirrored Langevin Dynamics}\label{sec:stochastic_mld}
\begin{algorithm}[t]
\caption{Stochastic Mirrored Langevin Dynamics (SMLD)}\label{alg:SMLD}
\begin{algorithmic}[1]
\Require{Target distribution $\dV$ where $V = \sum_{i=1}^NV_i$, step-sizes $\beta^t$, batch-size $b$}
    \State Find $W_i$ such that $e^{-NW_i} \propto \nh \# e^{-NV_i}$ for all $i$. 
    \For{$t \gets 0,1,\cdots, T-1$}
        \State Pick a mini-batch $B$ of size $b$ uniformly at random.
        \State Update $\y^{t+1} = \y^{t} -  \frac{\beta^tN}{b}\sum_{i\in B}\nabla W_i  (\y^t)+ \sqrt{2\beta^t}  \bm{\xi}^t $ 
        \State $\x^{t+1} = \nhstar (\y^{t+1})$ \Comment{Update only when necessary.}
        %
     \EndFor
\end{algorithmic}
\Return{$\x^T$}
\end{algorithm}

We have thus far only considered \textbf{deterministic} methods based on exact gradients. In practice, however, evaluating gradients typically involves one pass over the full data, which can be time-consuming in large-scale applications. In this section, we turn attention to the \textbf{mini-batch} setting, where one can use a small subset of data to form stochastic gradients. 


Toward this end, we assume:
\begin{assumption}[Primal Decomposibility]\label{ass:primal_decomposability}
The target distribution $\dV$ admits a decomposable structure $V = \sum_{i=1}^N V_i$ for some functions $V_i$.
\end{assumption}

Consider the following common scheme in obtaining stochastic gradients. Given a batch-size $b$, we randomly pick a mini-batch $B$ from $\{ 1, 2, \dots, N\}$ with $|B| = b$, and form an unbiased estimate of $\nabla V$ by computing
\beq
\tilde{\nabla}V \coloneqq \frac{N}{b} \sum_{i\in B} \nabla V_i. \label{eq:stochastic_gradient}
\eeq
The following lemma asserts that exactly the same procedure can be carried out in the dual.
\begin{lemma}\label{lem:unbiasedness}
Assume that $h$ is 1-strongly convex. For i = $1,2,...,N,$ let $W_i$ be such that
\beq \label{eq:dual_stochastic_gradient}
e^{-NW_i} = \nh \# \frac{e^{-NV_i}}{ \int e^{-NV_i}}.
\eeq
Define $W \coloneqq \sum_{i=1}^N W_i$ and $\tilde{\nabla}W \coloneqq \frac{N}{b} \sum_{i\in B} \nabla W_i$, where $B$ is chosen as in \eqref{eq:stochastic_gradient}. Then:
\vspace{-2mm}
\begin{enumerate}
\item Primal decomposibility implies dual decomposability: There is a constant $C$ such that $e^{-(W+C)} = \nh \# e^{-V}$.
\vspace{-2mm}
\item For each $i$, the gradient $\nabla W_i$ depends only on $\nabla V_i$ and the mirror map $h$.
\vspace{-2mm}
\item The gradient estimate is unbiased: $\EE \tilde{\nabla}W = \nabla W$.
\vspace{-2mm}
\item The dual stochastic gradient is more accurate: $\EE \| \tilde{\nabla}W - \nabla W \|^2 \leq \EE \| \tilde{\nabla}V - \nabla V \|^2$.
\end{enumerate}
\vspace{-2mm}
\end{lemma}
\begin{proof}\vspace{-3mm}
See \textbf{Appendix \ref{app:dual_decomposability}}.
\end{proof}\vspace{-2mm}
\textbf{Lemma \ref{lem:unbiasedness}} furnishes a template for the mini-batch extension of MLD. The pseudocode is detailed in \textbf{Algorithm \ref{alg:SMLD}}, whose convergence rate is given by the next theorem.


\begin{theorem}\label{thm:dirichlet_provable}
Let $\dV$ be a distribution satisfying \textbf{Assumption \ref{ass:primal_decomposability}}, and $h$ a 1-strongly convex mirror map. Let $\sigma^2 \coloneqq \EE \|  \tilde{\nabla}V - \nabla V \|^2$ be the variance of the stochastic gradient of $V$ in \eqref{eq:stochastic_gradient}. Suppose that the corresponding dual distribution $\dW = \nh \# \dV$ satisfies $LI \succeq \nabla^2 W \succeq 0$. Then, applying SMLD 
with constant step-size $\beta^t = \beta$ yields\footnote{Our guarantee is given on a randomly chosen iterate from $\{\x^1, \x^2, ..., \x^T\}$, instead of the final iterate $\x^T$. In practice, we observe that the final iterate always gives the best performance, and we will ignore this minor difference in the theorem statement.}:
\beq \label{eq:smld_rate}
D\l(\x^T \| \dV\r) \leq \sqrt{\frac{2\Wcal^2_2\l( \y^0, \dW\r)  \l( Ld+\sigma^2 \r) }{T}} = O\l(\sqrt{\frac{Ld+\sigma^2}{T}}\r),
\eeq
provided that $\beta \leq \min \left\{  \left[2T \Wcal^2_2\l( \y^0, \dW\r)  \l( Ld+\sigma^2 \r) \right]^{-\frac{1}{2}}, \frac{1}{L} \right\}$.
\end{theorem}
\begin{proof}
See \textbf{Appendix \ref{app:dual_convergence}}.
\end{proof}

\begin{example}[SMLD for Dirichlet Posteriors]\label{exp:diri_smld}
\emph{For the case of Dirichlet posteriors, we have seen in \eqref{eq:dirichlet_dual} that the corresponding dual distribution satisfies $(N+\Gamma)I \succeq \nabla^2 W \succ 0$, where $N \coloneqq \sum_{\ell=1}^{d+1} n_\ell$ and $\Gamma \coloneqq \sum_{\ell=1}^{d+1} \alpha_\ell$. Furthermore, it is easy to see that the stochastic gradient $\tilde{\nabla}W$ can be efficiently computed (see \textbf{Appendix \ref{app:diri_dual_stochastic}}):
\beq \label{eq:dual_W_i_dirichlet_minibatch}
\tilde{\nabla}W (\y)_\ell \coloneqq \frac{N}{b}\sum_{i\in B} \nabla W_i (\y)_\ell = -  \l(\frac{Nm_\ell}{b} + \alpha_\ell \r)  +  \l(N+ \Gamma\r) \frac{e^{y_\ell}}{1+\sum_{k=1}^d  e^{y_k} }, 
\eeq
where $m_\ell$ is the number of observations of category $\ell$ in the mini-batch $B$. As a result, \textbf{Theorem \ref{thm:dirichlet_provable}} states that SMLD achieves
\beq \nn
D\l(\x^T \| \dV\r) \leq \sqrt{\frac{2\Wcal^2_2\l( \y^0, \dW\r)  \Big( (N+\Gamma)(d+1)+\sigma^2 \Big) }{T}} = O\l(\sqrt{\frac{(N+\Gamma)d + \sigma^2}{T}}\r)
\eeq
with a constant step-size. \hfill $\square$
}
\end{example}\vspace{-3mm}

\section{Experiments}\label{sec:experiment}
We conduct experiments with a two-fold purpose. First, we use a low-dimensional synthetic data, where we can evaluate the total variation error by comparing histograms, to verify the convergence rates in our theory. Second, We demonstrate that the SMLD, modulo a necessary modification for resolving numerical issues, outperforms state-of-the-art first-order methods on the Latent Dirichlet Allocation (LDA) application with Wikipedia corpus.

\subsection{Synthetic Experiment for Dirichlet Posterior}\label{sec:experiment-synthetic}
We implement the deterministic MLD for sampling from an 11-dimensional Dirichlet posterior \eqref{eq:dirichlet_distribution} with $n_1 = \num[group-separator={,}]{10000},$ $n_2 = n_3 = 10$, and $n_4 = n_5 = \cdots = n_{11} = 0$, which aims to capture the sparse nature of real observations in topic modeling. We set $\alpha_\ell = 0.1$ for all $\ell$. 

As a baseline comparison, we include the Stochastic Gradient Riemannian Langevin Dynamics (SGRLD) \cite{patterson2013stochastic} with the expanded-mean parametrization. SGRLD is a tailor-made first-order scheme for simplex constraints, and it remains one of the state-of-the-art algorithms for LDA.
~For fair comparison, we use deterministic gradients for SGRLD.

We perform a grid search over the constant step-size for both algorithms, and we keep the best three for MLD and SGRLD.
~For each iteration, we build an empirical distribution by running \num[group-separator={,}]{2000000} independent trials, and we compute its total variation with respect to the histogram generated by the true distribution.

Figure~\ref{fig:aq1} reports the total variation error along the first dimension, where we can see that MLD outperforms SGRLD by a substantial margin.
~As dictated by our theory, all the MLD curves decay at the $O(T^{-\nicefrac{1}{2}})$ rate until they saturate at the dicretization error level. In contrast, SGRLD lacks non-asymptotic guarantees, and there is no clear convergence rate we can infer from Figure~\ref{fig:aq1}.

The improvement along all other dimensions (i.e., topics with less observations) are even more significant; see \textbf{Appendix \ref{app:experiment_synthetic}}.

\begin{figure}[t]
    \centering
    \subfigure[{\label{fig:aq1}}Synthetic data, first dimension.]{{\includegraphics[keepaspectratio=true,scale=0.22]{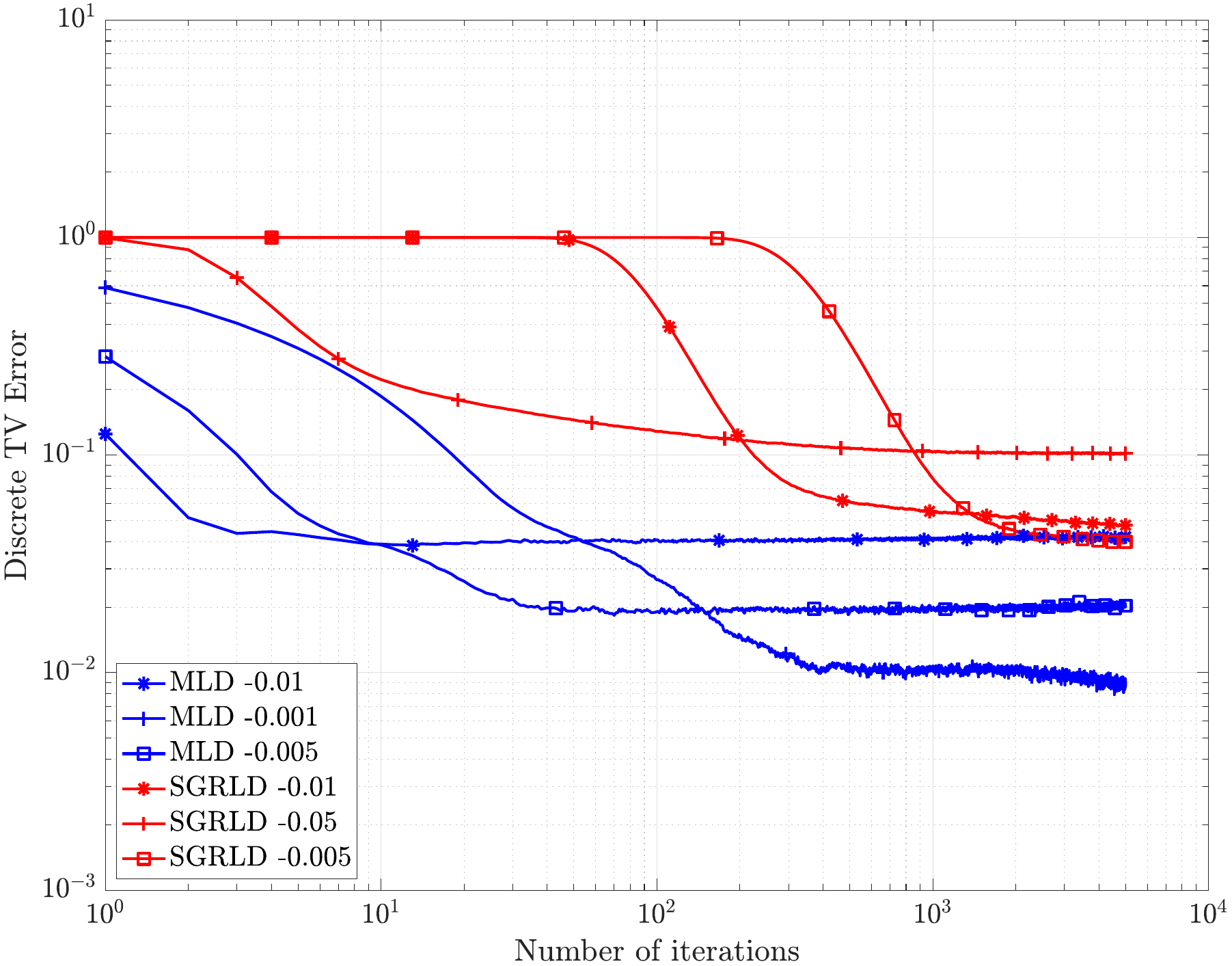} }}%
    \hspace{2mm}
    \subfigure[{\label{fig:bq1}}LDA on Wikipedia corpus.]{{\includegraphics[keepaspectratio=true,scale=0.22]{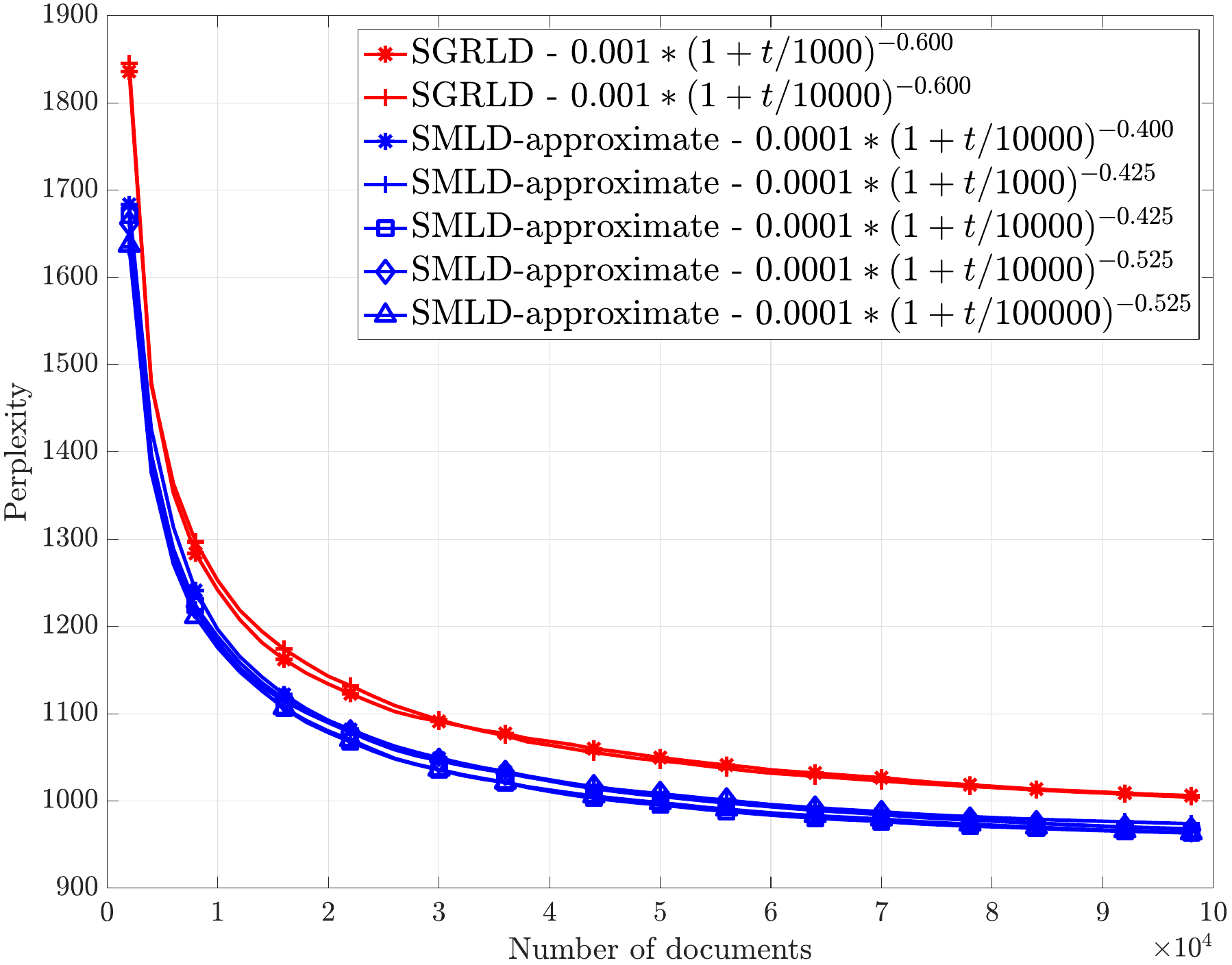} }}%
    \label{fig:adv}\vspace{-8mm}
\end{figure}
\subsection{Latent Dirichlet Allocation with Wikipedia Corpus}\label{sec:experiment-LDA}
An influential framework for topic modeling is the Latent Dirichlet Allocation (LDA) \cite{blei2003latent}, which, given a text collection, requires to infer the posterior word distributions without knowing the exact topic for each word. The full model description is standard but somewhat convoluted; we refer to the classic \cite{blei2003latent} for details.

Each topic $k$ in LDA determines a word distribution $\bm\pi_k$, and suppose there are in total $K$ topics and $W+1$ words. The variable of interest is therefore $\bm\pi \coloneqq (\bm\pi_1, \bm\pi_2, ..., \bm\pi_K) \in \Delta_{W} \times \Delta_{W}  \times \cdots \Delta_{W} $. Since this domain is a Cartesian product of simplices, we propose to use $\tilde{h}(\bm\pi) \coloneqq \sum_{k=1}^K h(\bm\pi_k)$, where $h$ is the entropic mirror map \eqref{eq:entropic_mirror}, for SMLD. It is easy to see that all of our computations for Dirichlet posteriors generalize to this setting.

\subsubsection{Experimental Setup}\label{sec:experiment-LDA_setup}
We implement the SMLD for LDA on the Wikipedia corpus with \num[group-separator={,}]{100000} documents, and we compare the performance against the SGRLD \cite{patterson2013stochastic}.
~In order to keep the comparison fair, we adopt exactly the same setting as in \cite{patterson2013stochastic}, including the model parameters, the batch-size, the Gibbs sampler steps, etc. See Section 4 and 5 in \cite{patterson2013stochastic} for omitted details.

Another state-of-the-art first-order algorithm for LDA is the SGRHMC in \cite{ma2015complete}, for which we skip the implementation, due to not knowing how the $\hat{B}_t$ was chosen in \cite{ma2015complete}. Instead, we will repeat the same experimental setting as \cite{ma2015complete} and directly compare our results versus the ones reported in \cite{ma2015complete}. See \textbf{Appendix \ref{app:experiment_SGRHMC}} for comparison against SGRHMC.

\subsubsection{A Numerical Trick and the SMLD-approximate Algorithm}
A major drawback of the SMLD in practice is that the stochastic gradients \eqref{eq:dual_W_i_dirichlet_minibatch} involve exponential functions, which are unstable for large-scale problems. For instance, in python, \texttt{np.exp(800) = inf}, whereas the relevant variable regime in this experiment extends to 1600. To resolve such numerical issues, we appeal to the linear approximation\footnote{One can also use a higher-order Taylor approximation for $\exp(\y)$, or add a small threshold $\exp(\y) \simeq \max\{\epsilon, 1+\y \}$ to prevent the iterates from going to the boundary. In practice, we observe that these variants do not make a huge impact on the performance.} $\exp(\y) \simeq \max\{0, 1+\y \}$. Admittedly, our theory no longer holds under such numerical tricks, and we shall not claim that our algorithm is provably convergent for LDA. Instead, the contribution of MLD here is to identify the dual dynamics associated with \eqref{eq:dirichlet_dual}, which would have been otherwise difficult to perceive. We name the resulting algorithm ``SMLD-approximate'' to indicate its heuristic nature.


\subsubsection{Results}\label{sec:experiment-LDA_results}
Figure~\ref{fig:bq1} reports the perplexity on the test data up to \num[group-separator={,}]{100000} documents, with the five best step-sizes we found via grid search for SMLD-approximate. For SGRLD, we use the best step-sizes reported in \cite{patterson2013stochastic}. 

From the figure, we can see a clear improvement, both in terms of convergence speed and the saturation level, of the SMLD-approximate over SGRLD. One plausible explanation for such phenomenon is that our MLD, as a simple unconstrained Langevin Dynamics, is less sensitive to discretization. On the other hand, the underlying dynamics for SGRLD is a more sophisticated Riemannian diffusion, which requires finer discretization than MLD to achieve the same level of approximation to the original continuous-time dynamics, and this is true even in the presence of noisy gradients and our numerical heuristics

\subsubsection*{Acknowledgments}

This project has received funding from the European Research Council (ERC) under the European Union's Horizon 2020 research and innovation programme (grant agreement n$^\circ$ 725594 - time-data).

%

\bibliographystyle{plain}
\bibliography{biblio}

\begin{thebibliography}{10}

\bibitem{ahn2012bayesian}
Sungjin Ahn, Anoop Korattikara, and Max Welling.
\newblock Bayesian posterior sampling via stochastic gradient fisher scoring.
\newblock In {\em Proceedings of the 29th International Coference on
  International Conference on Machine Learning}, pages 1771--1778, 2012.

\bibitem{beck2003mirror}
Amir Beck and Marc Teboulle.
\newblock Mirror descent and nonlinear projected subgradient methods for convex
  optimization.
\newblock {\em Operations Research Letters}, 31(3):167--175, 2003.

\bibitem{blei2003latent}
David~M Blei, Andrew~Y Ng, and Michael~I Jordan.
\newblock Latent dirichlet allocation.
\newblock {\em Journal of machine Learning research}, 3(Jan):993--1022, 2003.

\bibitem{brenier1987decomposition}
Yann Brenier.
\newblock D{\'e}composition polaire et r{\'e}arrangement monotone des champs de
  vecteurs.
\newblock {\em CR Acad. Sci. Paris S{\'e}r. I Math}, 305(19):805--808, 1987.

\bibitem{brenier1991polar}
Yann Brenier.
\newblock Polar factorization and monotone rearrangement of vector-valued
  functions.
\newblock {\em Communications on pure and applied mathematics}, 44(4):375--417,
  1991.

\bibitem{brosse17sampling}
Nicolas Brosse, Alain Durmus, \'Eric Moulines, and Marcelo Pereyra.
\newblock Sampling from a log-concave distribution with compact support with
  proximal langevin monte carlo.
\newblock In {\em Proceedings of the 2017 Conference on Learning Theory},
  volume~65 of {\em Proceedings of Machine Learning Research}, pages 319--342.
  PMLR, 07--10 Jul 2017.

\bibitem{bubeck2015sampling}
S{\'e}bastien Bubeck, Ronen Eldan, and Joseph Lehec.
\newblock Sampling from a log-concave distribution with projected langevin
  monte carlo.
\newblock {\em arXiv preprint arXiv:1507.02564}, 2015.

\bibitem{caffarelli1990localization}
Luis~A Caffarelli.
\newblock A localization property of viscosity solutions to the monge-ampere
  equation and their strict convexity.
\newblock {\em Annals of Mathematics}, 131(1):129--134, 1990.

\bibitem{caffarelli1992regularity}
Luis~A Caffarelli.
\newblock The regularity of mappings with a convex potential.
\newblock {\em Journal of the American Mathematical Society}, 5(1):99--104,
  1992.

\bibitem{caffarelli2000monotonicity}
Luis~A Caffarelli.
\newblock Monotonicity properties of optimal transportation and the fkg and
  related inequalities.
\newblock {\em Communications in Mathematical Physics}, 214(3):547--563, 2000.

\bibitem{chen2015convergence}
Changyou Chen, Nan Ding, and Lawrence Carin.
\newblock On the convergence of stochastic gradient mcmc algorithms with
  high-order integrators.
\newblock In {\em Advances in Neural Information Processing Systems}, pages
  2278--2286, 2015.

\bibitem{chen2014stochastic}
Tianqi Chen, Emily Fox, and Carlos Guestrin.
\newblock Stochastic gradient hamiltonian monte carlo.
\newblock In {\em International Conference on Machine Learning}, pages
  1683--1691, 2014.

\bibitem{cheng2017convergence}
Xiang Cheng and Peter Bartlett.
\newblock Convergence of langevin mcmc in kl-divergence.
\newblock In {\em Proceedings of Algorithmic Learning Theory}, volume~83 of
  {\em Proceedings of Machine Learning Research}, pages 186--211. PMLR, 07--09
  Apr 2018.

\bibitem{cheng2017underdamped}
Xiang Cheng, Niladri~S Chatterji, Peter~L Bartlett, and Michael~I Jordan.
\newblock Underdamped langevin mcmc: A non-asymptotic analysis.
\newblock {\em arXiv preprint arXiv:1707.03663}, 2017.

\bibitem{dalalyan2017theoretical}
Arnak~S Dalalyan.
\newblock Theoretical guarantees for approximate sampling from smooth and
  log-concave densities.
\newblock {\em Journal of the Royal Statistical Society: Series B (Statistical
  Methodology)}, 79(3):651--676, 2017.

\bibitem{dalalyan2017user}
Arnak~S Dalalyan and Avetik~G Karagulyan.
\newblock User-friendly guarantees for the langevin monte carlo with inaccurate
  gradient.
\newblock {\em arXiv preprint arXiv:1710.00095}, 2017.

\bibitem{de2014monge}
Guido De~Philippis and Alessio Figalli.
\newblock The monge--amp{\`e}re equation and its link to optimal
  transportation.
\newblock {\em Bulletin of the American Mathematical Society}, 51(4):527--580,
  2014.

\bibitem{ding2014bayesian}
Nan Ding, Youhan Fang, Ryan Babbush, Changyou Chen, Robert~D Skeel, and Hartmut
  Neven.
\newblock Bayesian sampling using stochastic gradient thermostats.
\newblock In {\em Advances in neural information processing systems}, pages
  3203--3211, 2014.

\bibitem{durmus2018analysis}
Alain Durmus, Szymon Majewski, and Bla{\.z}ej Miasojedow.
\newblock Analysis of langevin monte carlo via convex optimization.
\newblock {\em arXiv preprint arXiv:1802.09188}, 2018.

\bibitem{durmus2016stochastic}
Alain Durmus, Umut Simsekli, Eric Moulines, Roland Badeau, and Ga{\"e}l
  Richard.
\newblock Stochastic gradient richardson-romberg markov chain monte carlo.
\newblock In {\em Advances in Neural Information Processing Systems}, pages
  2047--2055, 2016.

\bibitem{dwivedi2018log}
Raaz Dwivedi, Yuansi Chen, Martin~J Wainwright, and Bin Yu.
\newblock Log-concave sampling: Metropolis-hastings algorithms are fast!
\newblock {\em arXiv preprint arXiv:1801.02309}, 2018.

\bibitem{frigyik2010introduction}
Bela~A Frigyik, Amol Kapila, and Maya~R Gupta.
\newblock Introduction to the dirichlet distribution and related processes.
\newblock {\em Department of Electrical Engineering, University of Washignton,
  UWEETR-2010-0006}, 2010.

\bibitem{kolesnikov2011mass}
Alexander~V Kolesnikov.
\newblock Mass transportation and contractions.
\newblock {\em arXiv preprint arXiv:1103.1479}, 2011.

\bibitem{krichene2017acceleration}
Walid Krichene and Peter~L Bartlett.
\newblock Acceleration and averaging in stochastic descent dynamics.
\newblock In {\em Advances in Neural Information Processing Systems}, pages
  6799--6809, 2017.

\bibitem{lan2016sampling}
Shiwei Lan and Babak Shahbaba.
\newblock Sampling constrained probability distributions using spherical
  augmentation.
\newblock In {\em Algorithmic Advances in Riemannian Geometry and
  Applications}, pages 25--71. Springer, 2016.

\bibitem{liu2016stochastic}
Chang Liu, Jun Zhu, and Yang Song.
\newblock Stochastic gradient geodesic mcmc methods.
\newblock In {\em Advances in Neural Information Processing Systems}, pages
  3009--3017, 2016.

\bibitem{luu2017sampling}
Tung Luu, Jalal Fadili, and Christophe Chesneau.
\newblock Sampling from non-smooth distribution through langevin diffusion.
\newblock 2017.

\bibitem{ma2015complete}
Yi-An Ma, Tianqi Chen, and Emily Fox.
\newblock A complete recipe for stochastic gradient mcmc.
\newblock In {\em Advances in Neural Information Processing Systems}, pages
  2917--2925, 2015.

\bibitem{mandelbrot1983fractal}
Benoit~B Mandelbrot.
\newblock {\em The fractal geometry of nature}, volume 173.
\newblock WH freeman New York, 1983.

\bibitem{mccann1995existence}
Robert~J McCann.
\newblock Existence and uniqueness of monotone measure-preserving maps.
\newblock {\em Duke Mathematical Journal}, 80(2):309--324, 1995.

\bibitem{mertikopoulos2018convergence}
Panayotis Mertikopoulos and Mathias Staudigl.
\newblock On the convergence of gradient-like flows with noisy gradient input.
\newblock {\em SIAM Journal on Optimization}, 28(1):163--197, 2018.

\bibitem{nemirovsky1983problem}
AS~Nemirovsky and DB~Yudin.
\newblock Problem complexity and method efficiency in optimization.
\newblock 1983.

\bibitem{patterson2013stochastic}
Sam Patterson and Yee~Whye Teh.
\newblock Stochastic gradient riemannian langevin dynamics on the probability
  simplex.
\newblock In {\em Advances in Neural Information Processing Systems}, pages
  3102--3110, 2013.

\bibitem{raginsky2012continuous}
Maxim Raginsky and Jake Bouvrie.
\newblock Continuous-time stochastic mirror descent on a network: Variance
  reduction, consensus, convergence.
\newblock In {\em Decision and Control (CDC), 2012 IEEE 51st Annual Conference
  on}, pages 6793--6800. IEEE, 2012.

\bibitem{rockafellar2015convex}
Ralph~Tyrell Rockafellar.
\newblock {\em Convex analysis}.
\newblock Princeton university press, 1970.

\bibitem{simsekli2016stochastic}
Umut Simsekli, Roland Badeau, Taylan Cemgil, and Ga{\"e}l Richard.
\newblock Stochastic quasi-newton langevin monte carlo.
\newblock In {\em International Conference on Machine Learning}, pages
  642--651, 2016.

\bibitem{villani2003topics}
C{\'e}dric Villani.
\newblock {\em Topics in optimal transportation}.
\newblock Number~58. American Mathematical Soc., 2003.

\bibitem{villani2008optimal}
C{\'e}dric Villani.
\newblock {\em Optimal transport: old and new}, volume 338.
\newblock Springer Science \& Business Media, 2008.

\bibitem{welling2011bayesian}
Max Welling and Yee~W Teh.
\newblock Bayesian learning via stochastic gradient langevin dynamics.
\newblock In {\em Proceedings of the 28th International Conference on Machine
  Learning (ICML-11)}, pages 681--688, 2011.

\bibitem{xu2018accelerated}
Pan Xu, Tianhao Wang, and Quanquan Gu.
\newblock Accelerated stochastic mirror descent: From continuous-time dynamics
  to discrete-time algorithms.
\newblock In {\em International Conference on Artificial Intelligence and
  Statistics}, pages 1087--1096, 2018.

\end{thebibliography}


\begin{thebibliography}{1}

\bibitem{ma2015complete}
Yi-An Ma, Tianqi Chen, and Emily Fox.
\newblock A complete recipe for stochastic gradient mcmc.
\newblock In {\em Advances in Neural Information Processing Systems}, pages
  2917--2925, 2015.

\end{thebibliography}

\newpage

\appendix
\numberwithin{equation}{section}

\section{Proof of \textbf{Theorem \ref{thm:convergence_TV}}} \label{app:proof_amld_TV}
We first focus on the convergence for total variation and relative entropy, since they are in fact quite trivial. The proof for the 2-Wasserstein distance requires a bit more work.

\subsection{Total Variation and Relative Entropy}
Since $h$ is strictly convex, $\nh$ is one-to-one, and hence 
\begin{align*}
d_{\textup{TV}}(\nabla h\# \mu_1, \nabla h\# \mu_2)
&= \frac{1}{2}\sup_{E} | \nabla h\# \mu_1(E) - \nabla h\# \mu_2(E)|  \\
&= \frac{1}{2}\sup_{E} \left| \mu_1\big( \nh^{-1}(E) \big) - \mu_2 \big(\nh^{-1}(E)\big) \right| \\
&=d_{\textup{TV}}(\mu_1, \mu_2).
\end{align*}

On the other hand, it is well-known that applying a one-to-one mapping to distributions leaves the relative entropy intact. Alternatively, we may also simply write (letting $\nu_i = \nh \# \mu_i$):
\begin{align*}
D(\nu_1 \|  \nu_2) &= \int \log \frac{\dnu_1}{\dnu_2} \dnu_1 \\
&=  \int \log  \l(\frac{\dnu_1}{\dnu_2}\circ \nh \r) \dmu_1 && \text{by \eqref{eq:push_forward} below}\\
&= \int \log  \frac{\dmu_1}{\dmu_2}  \dmu_1  && \text{by \eqref{eq:Monge-Ampere}}\\
&= D(\mu_1 \|  \mu_2)
\end{align*}

The ``in particular'' part follows from noticing that $\y^t \sim \nh \# \x^t$ and $\Y_\infty \sim \nh \#\X_\infty$.

\subsection{2-Wasserstein Distance}
Now, let $h$ be $\rho$-strongly convex. The most important ingredient of the proof is \textbf{Lemma \ref{lem:duality_wasserstein}} below, which is conceptually clean. Unfortunately, for the sake of rigor, we must deal with certain intricate regularity issues in the Optimal Transport theory. If the reader wishes, she/he can simply assume that the quantities \eqref{eq:primal_bregman_wasserstein} and \eqref{eq:dual_bregman_wasserstein} below are well-defined, which is always satisfied by any practical mirror map, and skip all the technical part about the well-definedness proof.

For the moment, assume $h \in \mathcal{C}^5$; the general case is given at the end. Every convex $h$ generates a Bregman divergence via $B_h (\x, \x') \coloneqq h(\x) - h(\x') - \ip{\nh (\x')}{\x-\x'}$. 
~The following key lemma allows us to relate guarantees in $\Wcal_2$ between $\x^t$'s and $\y^t$'s. It can be seen as a generalization of the classical duality relation \eqref{eq:bregman_fenchel} in the space of probability measures.
\begin{lemma}[Duality of Wasserstein Distances]\label{lem:duality_wasserstein}
Let $\mu_1$, $\mu_2$ be probability measures satisfying \textbf{Assumptions \ref{ass:distributions_moment}} and \textbf{\ref{ass:distributions_hausdorff}}. If $h$ is $\rho$-strongly convex and $\mathcal{C}^5$, then the \eqref{eq:primal_bregman_wasserstein} and \eqref{eq:dual_bregman_wasserstein} below are well-defined:  
\beq \label{eq:primal_bregman_wasserstein}
\Wcal_{B_h}(\mu_1, \mu_2) \coloneqq \inf_{T: T\#\mu_1 = \mu_2} \int B_h \l(\x, T(\x) \r) \drm \mu_1(\x)
\eeq
and (notice the exchange of inputs on the right-hand side)
\beq \label{eq:dual_bregman_wasserstein}
\Wcal_{B_\hstar}(\nu_1, \nu_2) \coloneqq \inf_{T: T\#\nu_1 = \nu_2} \int B_\hstar \l(T(\y), \y \r) \drm \nu_1(\y).
\eeq
Furthermore, we have
\beq \label{eq:duality_wasserstein}
\Wcal_{B_h}(\mu_1, \mu_2) = \Wcal_{B_\hstar}(\nabla h \#\mu_1, \nabla h \#\mu_2). 
\eeq
\end{lemma}

\vspace{8mm}

Before proving the lemma, let us see that the relation in $\Wcal_2$ is a simple corollary of \textbf{Lemma \ref{lem:duality_wasserstein}}. Since $h$ is $\rho$-strongly convex, it is classical that, for any $\x$ and $\x'$,
\begin{align}\label{eq:bregman_fenchel}
\frac{\rho}{2} \| \x - \x'\|^2 \leq B_h(\x, \x' ) = B_\hstar(\nh(\x'), \nh(\x))  \leq \frac{1}{2\rho}\| \nh(\x) - \nh(\x')\|^2.
\end{align}
Using \textbf{Lemma \ref{lem:duality_wasserstein}} and the fact that $\y^t \sim \nh \# \x^t$ and $\Y_\infty \sim \nh\# \X_\infty$, we conclude $\Wcal_2(\x^t, \X_\infty) \leq \frac{1}{\rho}\Wcal_2(\y^t, \X_\infty)$. It hence remains to prove \textbf{Lemma \ref{lem:duality_wasserstein}} when $h \in \mathcal{C}^5$.

\subsubsection{Proof of \textbf{Lemma \ref{lem:duality_wasserstein}} When $h \in \mathcal{C}^5$} \label{app:h_C5}
We first prove that \eqref{eq:dual_bregman_wasserstein} is well-defined by verifying the sufficient conditions in \textbf{Theorem 3.6} of \cite{de2014monge}. Specifically, we will verify \textbf{(C0)}-\textbf{(C2)} in p.554 of \cite{de2014monge} when the transport cost is $B_\hstar$.

Since $h$ is $\rho$-strongly convex, $\nh$ is injective, and hence $\nhstar = (\nh)^{-1}$ is also injective, which implies that $\hstar$ is strictly convex. On the other hand, the strong convexity of $h$ implies $\nabla^2 \hstar \preceq \frac{1}{\rho}I$, and hence $B_\hstar$ is globally upper bounded by a quadratic function. 

We now show that the conditions \textbf{(C0)}-\textbf{(C2)} are satisfied. Since we have assumed $h\in \mathcal{C}^5$, we have $B_\hstar \in \mathcal{C}^4$. Since $B_\hstar$ is upper bounded by a quadratic function, the condition \textbf{(C0)} is trivially satisfied. On the other hand, since $\hstar$ is strictly convex, simple calculation reveals that, for any $\y'$, the mapping $\y \rightarrow \nabla_{\y'}B_{\hstar}(\y, \y')$ is injective, which is \textbf{(C1)}. Similarly, for any $\y$, the mapping $\y' \rightarrow \nabla_{\y}B_{\hstar}(\y, \y')$ is also injective, which is \textbf{(C2)}. By  \textbf{Theorem 3.6} in \cite{de2014monge}, \eqref{eq:dual_bregman_wasserstein} is well-defined.

We now turn to \eqref{eq:duality_wasserstein}, which will automatically establish the well-definedness of \eqref{eq:primal_bregman_wasserstein}. We first need the following equivalent characterization of $\nh \# \mu = \nu$ \cite{villani2008optimal}:
\beq \label{eq:push_forward}
\int f \dnu = \int f\circ \nh  \dmu
\eeq
for all measurable $f$. Using \eqref{eq:push_forward} in the definition of $\Wcal_{B_\hstar}$, we get
\begin{align*} 
\Wcal_{B_\hstar}(\nabla h \#\mu_1, \nabla h \#\mu_2) &= \inf_{T} \int B_\hstar \l(T(\y), \y \r) \drm \nabla h\# \mu_1(\y) \\
&= \inf_{T} \int B_\hstar \Big( (T\circ\nabla h)(\x), \nabla h(\x) \Big) \drm  \mu_1(\x),
\end{align*}where the infimum is over all $T$ such that $T\# (\nabla h\# \mu_1) = \nabla h\# \mu_2$. Using the classical duality $B_h(\x, \x') = B_\hstar(\nabla h(\x'), \nabla h(\x))$ and $\nabla h \circ \nabla \hstar (\x)= \x$, we may further write
\begin{align} \label{eq:hold5}
\Wcal_{B_\hstar}(\nabla h \#\mu_1, \nabla h \#\mu_2) &= \inf_{T} \int B_h \Big( \x, (\nhstar \circ T \circ \nh ) (\x) \Big) \drm \mu_1(\x) 
\end{align}where the infimum is again over all $T$ such that $T\# (\nabla h\# \mu_1) = \nabla h\# \mu_2$. In view of \eqref{eq:hold5}, the proof would be complete if we can show that $T\# (\nabla h\# \mu_1) = \nabla h\# \mu_2$ if and only if $(\nhstar \circ T \circ \nh ) \# \mu_1 = \mu_2$. 

For any two maps $T_1$ and $T_2$, we claim that 
\beq
(T_1 \circ T_2) \# \mu = T_1\#\l(  T_2 \#\mu \r). \label{eq:equivalence_push_forward}
\eeq
Indeed, for any Borel set $E$, we have, by definition of the push-forward,
\begin{align*}
(T_1 \circ T_2) \# \mu (E) &= \mu \big( (T_1 \circ T_2)^{-1} (E) \big) \\
&= \mu \big( (T_2^{-1} \circ T_1^{-1}) (E) \big).
\end{align*}On the other hand, recursively applying the definition of push-forward to $T_1\#\l(  T_2 \#\mu \r)$ gives
\begin{align*}
T_1\#\l(  T_2 \#\mu \r)(E) &=  T_2 \#\mu \big( T^{-1}(E) \big)  \\
&= \mu \big(  (T_2^{-1} \circ T_1^{-1}) (E) \big)
\end{align*}which establishes \eqref{eq:equivalence_push_forward}. 

Assume that $T \#(\nh \# \mu_1) = \nh \# \mu_2$. Then we have
\begin{align*}
(\nhstar \circ T \circ \nh ) \# \mu_1 &= \nhstar \# (T \#(\nh \# \mu_1)) && \textup{by \eqref{eq:equivalence_push_forward}} \\
&= \nhstar \# (\nh \#\mu_2) && \textup{since $T \#(\nh \# \mu_1) = \nh \# \mu_2$} \\
&= (\nhstar \circ \nh)\#\mu_2 && \textup{by \eqref{eq:equivalence_push_forward} again} \\
&= \mu_2.
\end{align*}
On the other hand, if $(\nhstar \circ T \circ \nh ) \# \mu_1 = \mu_2$, then composing both sides by $\nh$ and using \eqref{eq:equivalence_push_forward} yields $T\# (\nabla h\# \mu_1) = \nabla h\# \mu_2$, which finishes the proof.

\subsubsection{When $h$ is only $\mathcal{C}^2$}

When $h$ is only $\mathcal{C}^2$, we will directly resort to \eqref{eq:bregman_fenchel}. Let $T$ be any map such that $T\# (\nabla h\# \mu_1) = \nabla h\# \mu_2$, and consider the optimal transportation problem $\inf_{T} \int \|  \y- T(\y) \|^2 \drm \nh\#\mu_1(\y)$. By \eqref{eq:bregman_fenchel} and \eqref{eq:push_forward}, we have
\begin{align}
\inf_{T} \int \|  \y- T(\y) \|^2 \drm \nh\#\mu_1(\y) &= \inf_{T} \int \|  \nh(\x)- (T\circ\nh)(\x)) \|^2 \drm \mu_1(\x) \nn \\
&\geq \rho^2 \inf_{T} \int \|  \x- (\nhstar \circ T\circ\nh)(\x)) \|^2 \drm \mu_1(\x) \nn
\end{align}where the infimum is over all $T$ such that $T\# (\nabla h\# \mu_1) = \nabla h\# \mu_2$. But as proven in \textbf{Appendix \ref{app:h_C5}}, this is equivalent to $(\nhstar \circ T\circ\nh) \# \mu_1 = \mu_2$. The proof is finished by noting $\y^t \sim \nh \# \x^t$ and $\Y_\infty \sim \nh \#\X_\infty$.



\section{More Examples of Mirror Map and their Dual Distributions}\label{app:mirror_maps}

In this section, we present more instances of mirror map other than on the simplex, and their corresponding dual distributions.

\subsection{Mirror map on the hypercube}

On the hypercube $[-1,1]^d$, a possible mirror map is \[h(\x) = \frac{1}{2} \sum_{i=1}^d \Big((1+x_i)\log(1+x_i) + (1-x_i)\log(1-x_i)\Big).\]

It can easily be shown that
\[
\frac{\partial h}{\partial x_i} = \text{arctanh}(x_i), \ \ \ \ \ \ \ \frac{\partial^2h}{\partial x_i \partial x_j} = \frac{\delta_{ij}}{1 - x_i^2},  \ \ \ \ \ \ \  \frac{\partial h^*}{\partial y_i} = \tanh(y_i)
\]
which implies that $h$ is $1$-strongly convex.

Since the Hessian matrix of $h$ is diagonal, we have:
\[
\log \det \nabla^2 h(\x) = \sum_{i=1}^d \log\left(\frac{1}{1-x_i^2}\right).
\]
Then, using the definition of $W$, we obtain
\begin{align*}
W(\y) &= V \circ \nabla h^*(\y) + \sum_{i=1}^d \log\left(\frac{1}{1-\tanh^2(y_i)}\right) \\
&= V \circ \nabla h^*(\y) + \sum_{i=1}^d \log\left(\frac{4}{\exp(2y_i) + 2 + \exp(-2y_i)}\right) \\
&= V \circ \nabla h^*(\y) + \sum_{i=1}^d \log\left(\frac{1}{2}(1 + \cosh(2y_i))\right).
\end{align*}

\subsection{Mirror map on the Euclidean ball}

On the unit ball $\{\x \in \mathbb{R}^d : \|\x\| \leq 1\}$, where $\|\cdot\|$ denote the Euclidean norm, a possible mirror map is
\[
h(\x) = -\log(1-\|\x\|)  - \|\x\|.
\]

We can compute:
\[
\frac{\partial h}{\partial x_i} = \frac{x_i}{1 - \|\x\|}, \ \ \ \ \ \ \ \frac{\partial^2h}{\partial x_i \partial x_j} = \frac{\delta_{ij}}{1 - \|\x\|} + \frac{x_ix_j}{\|\x\|(1-\|\x\|)^2},  \ \ \ \ \ \ \  \frac{\partial h^*}{\partial y_i} = \frac{y_i}{1 + \|\y\|}.
\]

The Hessian matrix can thus be written as $\nabla^2 h(\x) = \frac{1}{1-\|\x\|} {I} + \frac{1}{\|\x\|(1 - \|\x\|)^2} \x \x^T$, where ${I}$ is the identity matrix. Invoking the matrix determinant lemma, we get
\[
\det(\nabla^2 h(\x)) = \left(1 + \frac{\x^\top \x}{\|\x\|(1-\|\x\|))}\right) \det\left(\frac{1}{1 - \|\x\|} {I}\right) = \left(\frac{1}{1 - \|\x\|}\right)^{d+1}.
\]

We thus obtain:
\begin{align*}
W(\y) &= V \circ \nabla h^*(\y) - (d+1)\log \left(1 - \frac{\|\y\|}{1 + \|\y\|} \right) \\
&= V \circ \nabla h^*(\y) + (d+1) \log\left(1 + \|\y\|\right).
\end{align*}

\section{Proof of \textbf{Thereom \ref{thm:existence_good_mirror_map}}}\label{app:proof_existence}
In previous sections, we are given a target distribution $e^{-V}$ and a mirror map $h$, and we derive the induced distribution $e^{-W}$ through the Monge-Amp\`ere equation \eqref{eq:Monge-Ampere}. The high-level idea of this proof is to reverse the direction: We start with two good distributions $e^{-V}$ and $e^{-W}$, and we invoke deep results in Optimal Transport to deduce the existence of a good mirror map $h$.

First, notice that if $V$ has bounded domain, then the strong convexity of $V$ implies $V > -\infty$. Along with the assumption that $V$ is bounded away from $+\infty$ in the interior, we see that $e^{-V}$ is bounded away from $0$ and $+\infty$ in the interior of support.

Let $\drm\nu (\x) = e^{-W(\x)} \drm \x$ be any distribution such that $\nabla^2 W \succeq I$. By Brenier's polarization theorem \citep{brenier1987decomposition, brenier1991polar} and \textbf{Assumption \ref{ass:distributions_moment}}, \textbf{\ref{ass:distributions_hausdorff}}, there exists a convex function $\hstar$ whose gradient solves the $\Wcal_2\l(\nu, \mu \r)$ optimal transportation problem. Caffarelli's regularity theorem \citep{caffarelli1990localization, caffarelli1992regularity, caffarelli2000monotonicity} then implies that the Brenier's map $h^\star$ is in $\mathcal{C}^2$. Finally, a slightly stronger form of Caffarelli's contraction theorem \citep{kolesnikov2011mass} asserts:
\beq \label{eq:hold_sec_asymme6}
\nabla^2 \hstar \preceq \frac{1}{m} I,
\eeq
which implies $h= (\hstar)^\star$ is $m$-strongly convex.
 
Let us consider the discretized MLD \eqref{eq:asymmetric_mirror_dynamics_discrete} corresponding to the mirror map $h$. Invoking \textbf{Theorem 3} of \cite{cheng2017convergence}, the convergence rate of the discretized Langevin dynamics $\y^T$ for $\mu$ is such that $D(\y^T \| \nu) = \tilde{O}\l(   \nicefrac{d}{T} \r)$, which in turn implies $\Wcal_2( \y^T, \nu) = \tilde{O}\l(  \sqrt{ \nicefrac{d}{{T}} } \r)$ and $\dTV( \y^T,  \nu) = \tilde{O}\l(  \sqrt{ \nicefrac{d}{{T}} } \r).$ \textbf{Theorem \ref{thm:convergence_TV}} then completes the proof.

%

\section{Proof of \textbf{Lemma \ref{lem:W_simplex}}} \label{app:W_simplex}
Straightforward calculations in convex analysis shows
\begin{align}
&\frac{\partial h}{\partial x_i}= \log \frac{x_i}{x_{d+1}}, \quad \quad  \quad \quad\frac{\partial^2 h}{\partial x_i \partial x_j} =\delta_{ij} x_i^{-1} + x_{d+1}^{-1}, \nonumber \\
&\hstar(\y) = \log\l( 1+  \sum_{i=1}^d e^{y_i}  \r), \quad \frac{\partial \hstar}{\partial y_i} = \frac{e^{y_i} }{1+\sum_{i=1}^de^{y_i}   }, \label{eq:formulas_entropy}
\end{align}
which proves that $h$ is 1-strongly convex.

Let $\mu = \dV$ be the target distribution and define $\nu = \dW \coloneqq \nabla h\# \mu$. By \eqref{eq:Monge-Ampere}, we have 
\begin{align}
W\circ \nh &= V + \log\det \nabla^2 h. \label{eq:MAhold}
\end{align}
Since $\nabla^2 h (\x) = \textup{diag}[x^{-1}_i] + x_{d+1}^{-1} \mathbbm{1}\mathbbm{1}^\top$ where $\mathbbm{1}$ is the all 1 vector, the well-known matrix determinant lemma ``$\det (A + \mathbf{u} \mathbf{v}^\top) = (1+ \mathbf{v}^\top A^{-1} \mathbf{u})\det A$'' gives
\begin{align}
\log\det\nabla^2 h(\x) &= \log \l( 1+   x_{d+1}^{-1}\sum_{i=1}^d x_i \r) \cdot \prod_{i=1}^d x^{-1}_i \nn\\
&= -\sum_{i=1}^{d+1} \log x_i = -\sum_{i=1}^{d} \log x_i  - \log \l(1- \sum_{i=1}^d x_i \r). \label{eq:section_amld_hold2}
\end{align}Composing both sides of \eqref{eq:MAhold} with $\nhstar$ and using \eqref{eq:formulas_entropy}, \eqref{eq:section_amld_hold2}, we then finish the proof by computing
\begin{align}
W(\y) &= V\circ \nhstar (\y)- \sum_{i=1}^d{y_i} + (d+1) \log \l(  1+ \sum_{i=1}^d e^{y_i} \r) \nonumber \\
&= V\circ \nhstar (\y)- \sum_{i=1}^d{y_i} + (d+1) \hstar (\y). \nn
\end{align}

\section{Proof of \textbf{Lemma \ref{lem:unbiasedness}}}\label{app:dual_decomposability}
The proof relies on rather straightforward computations.

\begin{enumerate}
\item In order to show $e^{-(W+C)} = \nabla h \# e^{-V}$ for some constant $C$, we will verify the Monge-Amp\`ere equation:
\begin{equation}\label{eq:app_proof_unbiased_hold0}
e^{-V} = e^{-(W\circ \nabla h +C)}\text{det}\nabla ^2 h
\end{equation}for $V = \sum_{i=1}^NV_i$ and $W = \sum_{i=1}^N W_i$, where $W_i$ is defined via \eqref{eq:dual_stochastic_gradient}. By \eqref{eq:dual_stochastic_gradient}, it holds that
\beq \label{eq:app_proof_unbiased_hold1}
\frac{1}{C_i}e^{-NV_i} = e^{-NW_i\circ \nabla h}\det\nabla^2 h, \quad C_i \coloneqq \frac{1}{\int e^{-NV_i}}.
\eeq
Multiplying \eqref{eq:app_proof_unbiased_hold1} for $i = 1, 2, ..., N$, we get
\beq \label{eq:app_proof_unbiased_hold2}
\prod_{i=1}^N \frac{1}{C_i} e^{-NV} = e^{-NW\circ \nh} \l(\det \nabla^2 h \r)^N.
\eeq
The first claim follows by taking the $N$\ts{th} root of \eqref{eq:app_proof_unbiased_hold2}.

\item The second claim directly follows by \eqref{eq:app_proof_unbiased_hold1}.

\item Trivial.

\item By \eqref{eq:app_proof_unbiased_hold0} and \eqref{eq:app_proof_unbiased_hold1} and using $\nabla \hstar \circ \nabla h(\x) = \x$, we get
\begin{align}
W_i &=  V_i \circ \nabla \hstar +  \frac{1}{N}   \log\det\nabla^2 h (\nabla \hstar) - \log C_i, \\
W &= V \circ \nhstar + \log\det\nabla^2 h (\nabla \hstar) - C,
\end{align}
which implies $N\nabla W_i  - \nabla W=   \nabla^2 \hstar \l( N\nabla V_i \circ \nabla \hstar  - \nabla V \circ \nabla \hstar\r)$. Since $h$ is 1-strongly convex, $\hstar$ is 1-Lipschitz gradient, and therefore the spectral norm of $\nabla^2 \hstar$ is upper bounded by 1. In the case of $b=1$, the final claim follows by noticing
\begin{align}
\EE \|  \tilde{\nabla}W - \nabla W   \|^2 &= \frac{1}{N}\sum_{i=1}^N \|  N\nabla W_i - \nabla W  \|^2 \\
&= \frac{1}{N}\sum_{i=1}^N  \|   \nabla^2 \hstar \l( N\nabla V_i \circ \nabla \hstar  - \nabla V \circ \nabla \hstar\r)  \|^2 \\
&\leq  \frac{\|  \nabla^2 \hstar \|_{\textup{spec}}^2}{N} \sum_{i=1}^N  \|N\nabla V_i \circ \nabla \hstar  - \nabla V \circ \nabla \hstar \|^2 \\
&\leq \EE \|  \tilde{\nabla}V - \nabla V   \|^2.
\end{align}The proof for general batch-size $b$ is exactly the same, albeit with more cumbersome notation.
\end{enumerate}

\section{Proof of \textbf{Theorem \ref{thm:dirichlet_provable}}}\label{app:dual_convergence}
The proof is a simple combination of the existing result in \cite{durmus2018analysis} and our theory in Section \ref{sec:amld}.

By \textbf{Theorem \ref{thm:convergence_TV}}, we only need to prove that the inequality \eqref{eq:smld_rate} holds for $D(\tilde{\y}^T\| \dW)$, where $\tilde{\y}^T$ is to be defined below. By assumption, $W$ is unconstrained and satisfies $LI \succeq \nabla^2 W \succeq 0$. By \textbf{Lemma~\ref{lem:unbiasedness}}, the stochastic gradient $\tilde{\nabla}W$ is unbiased and satisfies 
\beq \nn
\EE \| \tilde{\nabla}W - \nabla W\|^2 \leq \EE \| \tilde{\nabla}V - \nabla V\|^2 = \sigma^2.
\eeq

Pick a random index\footnote{The analysis in \cite{durmus2018analysis} provides guarantees on the probability measure $\nu_T \coloneqq \frac{1}{N}\sum_{t=1}^T \nu_t$ where $\y^t \sim \nu_t$. The $\tilde{\y}^T$ defined here has law $\nu_T$.} $t\in \{ 1, 2, ..., T\}$ and set $\tilde{\y}^T \coloneqq \y^t$. Then \textbf{Corollary 18} of~\cite{durmus2018analysis} with $D^2 = \sigma^2$ and $M_2 = 0$ implies $D(\tilde{\y}^T \| \dW) \leq \epsilon$, provided 
\beq
\beta \leq \min \left\{   \frac{\epsilon}{2\l(  Ld + \sigma^2   \r)}, \frac{1}{L}  \right\}, \quad T \geq \frac{\Wcal_2^2(\y^0, \dW)}{\beta\epsilon}.
\eeq
Solving for $T$ in terms of $\epsilon$ establishes the theorem.

\section{Stochastic Gradients for Dirichlet Posteriors}\label{app:diri_dual_stochastic}
In order to apply SMLD, one must have, for each term $V_i$, the corresponding dual $W_i$ defined via \eqref{eq:dual_stochastic_gradient}.
In this appendix, we derive a closed-form expression in the case of the Dirichlet posterior \eqref{eq:dirichlet_distribution}. 

Recall that the Dirichlet posterior \eqref{eq:dirichlet_distribution} consists of a Dirichlet prior and categorical data observations \cite{frigyik2010introduction}. Let $N \coloneqq \sum_{\ell=1}^{d+1} n_\ell$, where $n_\ell$ is the number of observations for category $\ell$, and suppose that the parameters $\alpha_\ell$'s are given. If the i\ts{th} data is in category $c_i \in \{1, 2, ..., d+1\}$, then we can define $V_i(\x) \coloneqq - \sum_{\ell=1}^{d+1} \mathbb{I}_{\{\ell = c_i\}}\log x_\ell - \frac{1}{N}\sum_{\ell=1}^{d+1} (\alpha_\ell - 1) \log x_\ell$ so that \textbf{Assumption \ref{ass:primal_decomposability}} holds. In view of \textbf{Lemma \ref{lem:W_simplex}}, The corresponding dual $W_i$ is, up to a constant, given by 
\beq \label{eq:dual_W_i_dirichlet}
W_i (\y) = - \sum_{\ell=1}^d \mathbb{I}_{ \{\ell=c_i\} } y_\ell -  \sum_{\ell=1}^d \frac{\alpha_\ell}{N} y_\ell + \hstar + \l(\sum_{\ell=1}^{d+1} \frac{\alpha_\ell}{N} \r) \hstar(\y).
\eeq
Similarly, if we take a mini-batch $B$ of the data with $|B| = b$, then 
\beq \label{eq:dual_W_i_dirichlet_minibatchh}
\frac{N}{b}\tilde{W} (\y) \coloneqq \frac{N}{b}\sum_{i\in B} W_i (\y) = - \sum_{\ell=1}^d \l(\frac{Nm_\ell}{b} + \alpha_\ell \r) y_\ell +  \l(N+ \sum_{\ell=1}^{d+1}{\alpha_\ell}\r) \hstar(\y), 
\eeq
where $m_\ell$ is the number of observations of category $\ell$ in the set $B$. Apparently, the gradient of \eqref{eq:dual_W_i_dirichlet_minibatchh} is \eqref{eq:dual_W_i_dirichlet_minibatch}.

\section{More on Experiments}

\subsection{Synthetic Data}\label{app:experiment_synthetic}
Figure~\ref{fig:app_1a} reports the total variation error along the 8\textsuperscript{th} dimension of the synthetic experiment in Section \ref{sec:experiment-synthetic}. Compared to Figure~\ref{fig:aq1} in the main text, it is evident that MLD achieves an even stronger performance than SGRLD, especially in the saturation error phase.

\subsection{Comparison against SGRHMC for Latent Dirichlet Allocation}\label{app:experiment_SGRHMC}
The only difference between the experimental setting of \cite{ma2015complete} and the main text is the number of topics (50 vs. 100). In this appendix, we run SMLD-approximate under the setting of \cite{ma2015complete} and directly compare against the results reported in \cite{ma2015complete}. We have also included the SGRLD as a baseline.

Figure~\ref{fig:app_1b} reports the perplexity on the test data. According to \cite{ma2015complete}, the best perplexity achieved by SGRHMC up to \num[group-separator={,}]{10000} documents is approximately \num[group-separator={,}]{1400}, which is worse than the 1323 by SMLD-approximate. Moreover, from Figure 3 of \cite{ma2015complete}, we see that the SGRHMC yields comparable performance as SGRLD for 2 out 3 independent runs, especially in the beginning phase, whereas the SMLD-approximate has sizeable lead over SGRLD at any stage of the experiment. The potential reason for this improvement is, similar to SGRLD, that the SGRHMC exploits the Riemannian Hamiltonian dynamics, which is more complicated than MLD and hence more sensitive to the discretization error.

\begin{figure}[t]
    \centering
    \subfigure[{\label{fig:app_1a}}Synthetic data, 8\textsuperscript{th} dimension.]{{\includegraphics[keepaspectratio=true,scale=0.23]{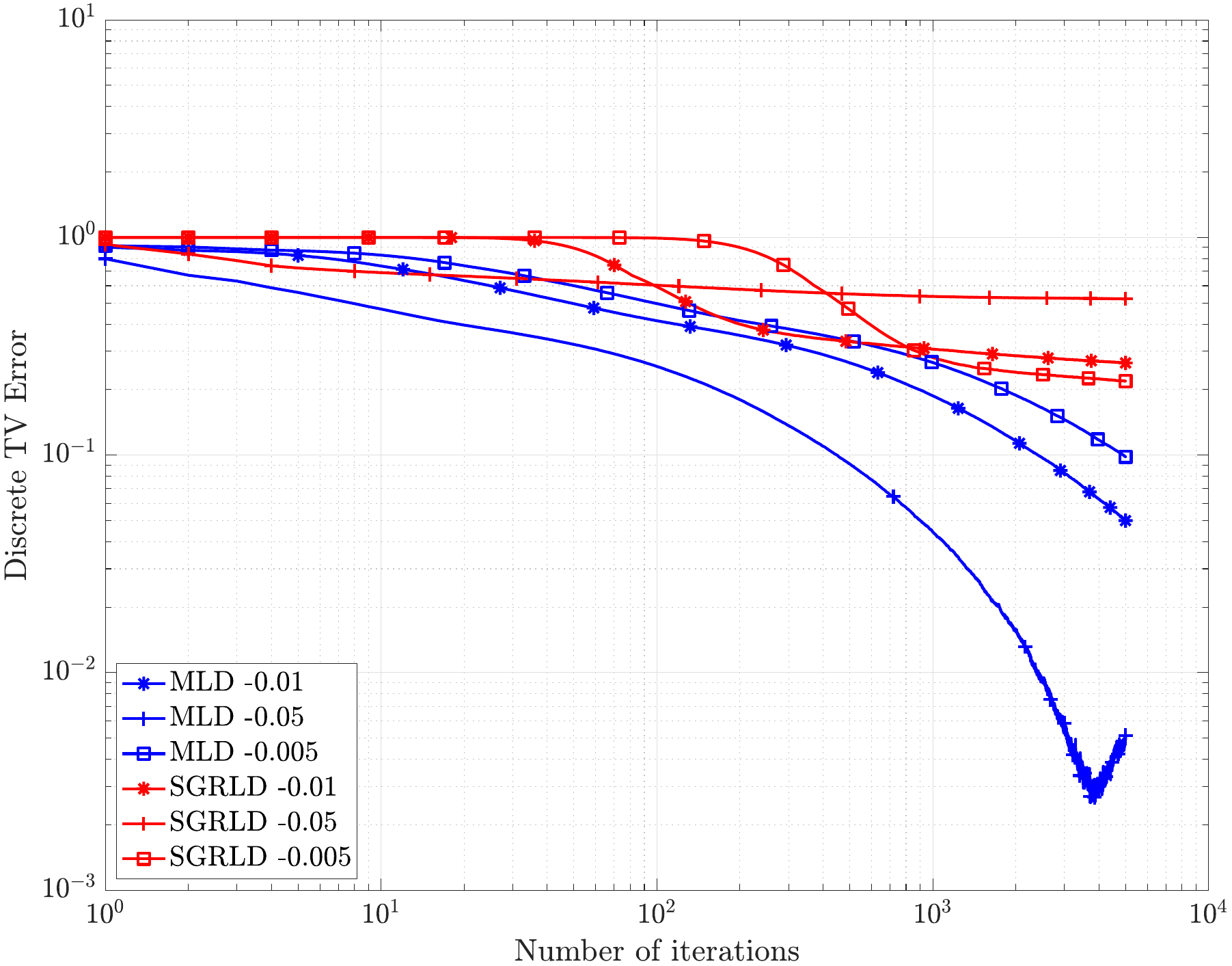} }}%
    \subfigure[{\label{fig:app_1b}}LDA on Wikipedia corpus.]{{\includegraphics[keepaspectratio=true,scale=0.23]{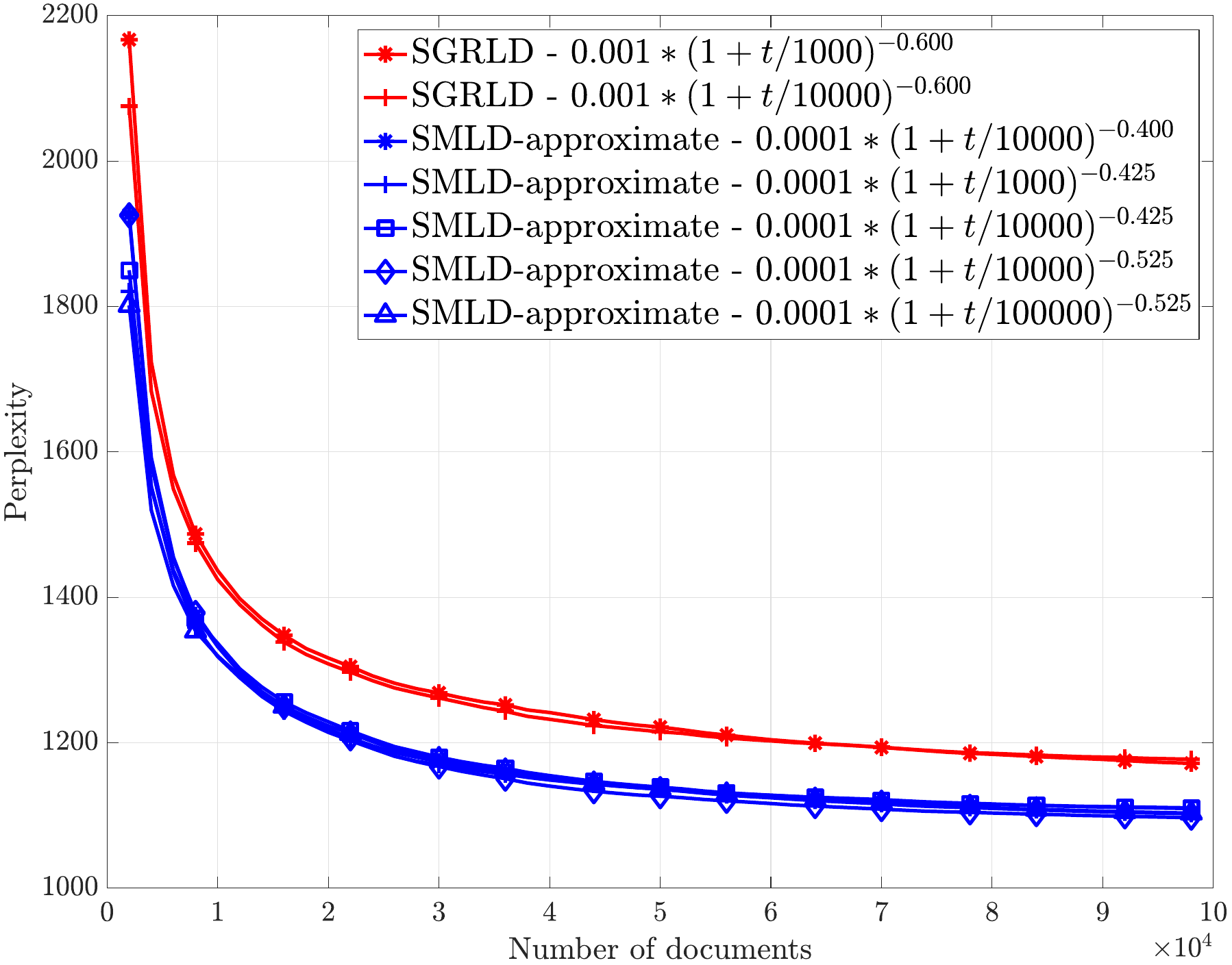} }}%
    \caption{LDA for Wikipedia, $50$ topics.}\
\end{figure}



%

\end{document}